\documentclass[article]{elsarticle} 
\usepackage{lineno,hyperref}
\usepackage{amsthm,amsfonts,amssymb,amsmath,mathrsfs}
\usepackage{mathtools}
\usepackage{pifont}
\usepackage{latexsym}
\usepackage{enumitem}
\usepackage{algorithm}
\usepackage{algorithmic}
\usepackage{diagbox}
\usepackage{multirow}
\newtheorem{theorem}{Theorem}
\newtheorem{lemma}[theorem]{Lemma}

\journal{Journal of \LaTeX\ Templates}

\makeatletter
\def\ps@pprintTitle{%
	\let\@oddhead\@empty
	\let\@evenhead\@empty
	\def\@oddfoot{\centerline{\thepage}}%
	\let\@evenfoot\@oddfoot}
\makeatother
\begin{document}

\begin{frontmatter}

\title{A novel extension of Generalized Low-Rank Approximation of Matrices based on multiple-pairs of transformations
}


\author[mymainaddress]{Soheil Ahmadi}

\author[mymainaddress]{Mansoor Rezghi\corref{mycorrespondingauthor}}
\cortext[mycorrespondingauthor]{Corresponding author}
\ead{rezghi@modares.ac.ir}

\address[mymainaddress]{Department of Computer Science, Tarbiat Modares University, Tehran, Iran}

\begin{abstract}
	
Dimensionality reduction is a main step in the learning process which plays an essential role in many applications. The most popular methods in this field like SVD, PCA, and LDA,  only can be applied to data with vector format. This means that for higher order data like matrices or more generally tensors, data should be fold to the vector format.
So, in this approach, the spatial relations of features are not considered and also the probability of over-fitting is increased. Due to these issues, in recent years some methods like Generalized low-rank approximation of matrices (GLRAM) and  Multilinear PCA (MPCA) are proposed which deal with the data in their own format. So, in these methods, the spatial relationships of features are preserved and the probability of overfitting could be fallen.
Also, their time and space complexities are less than vector-based ones. However, because of the fewer parameters, the search space in multilinear approach is much smaller than the search space of the vector-based approach. To overcome this drawback of multilinear methods like GLRAM, we proposed a new method which is a general form of GLRAM and by preserving the merits of it have a larger search space. Experimental results confirm the quality of the proposed method.
Also, applying this approach to the other multilinear dimensionality reduction methods like MPCA and MLDA is straightforward.

\end{abstract}


\end{frontmatter}



\section{Introduction}

Machine learning (ML) is one of the most important concepts in computer science which has many applications in the real world such as face recognition\cite{0-1}, image processing\cite{0-2}, criminal recognition\cite{0-3}, medical images\cite{0-4}, computer vision\cite{2-9}, data mining\cite{2-8}, etc.
In the literature of ML,  each sample is considered to be a vector. This means that in well-known  ML algorithms like logistic regression classifier, 
SVM \cite{0-5}, LDA\cite{0-6}, PCA\cite{0-7}, SVD\cite{0-8} and others,  
the input data in other formats like matrix or tensor should be folded to the vector format.


This folding causes two major problems. At the first, by converting a matrix or a tensor to a long (wide) vector, the number of free variables of any learning model will be increased sharply, which can make overfitting in the model.
Also, by vectorizing the spatial relationships of features for some data like images and videos are not
considered. In other words, each datum treated individually\cite{0-9}. For example, a grayscale image represented by $m\times n$ matrix in this approach will be reshaped to a vector with the size $nm$. Therefore, not only we have many free variables but also the local spatial relations among pixels of images are not considered.

After a while in order to tackle mentioned drawbacks of vector-based methods, another approach
is known multi-linear(tensor-based) learning has been proposed which consider the data in their original multidimensional format\cite{1-1}\cite{1-2}. In this approach, 
it is not necessary to reshape the data with multidimensional format anymore, and so the spatial relationships in data will be preserved\cite{0-9}. Furthermore, in contrast with vector-based methods, these multilinear methods have much fewer variables which can reduce the computational complexity and also the probability of overfitting\cite{1-0}.
In this approach different methods like Generalized low rank approximation(GLRAM) \cite{1-8}, multilinear PCA(MPCA)\cite{1-3},
 multilinear LDA(MLDA)\cite{1-4}, support tensor machine(STM)\cite{1-5}, have been investigated as tensor counterparts of SVD, PCA, LDA, SVM, respectively\cite{1-6}\cite{2-3}.
 
Despite the mentioned appropriate properties, these methods have some problems, either. 
The main problem of tensor-based methods is their limited search space which is only a subset of vector-based one. So, the probability of finding an optimal answer in such a small search space is less than the large feasible region in traditional methods by far. 
The problem of dimensionality reduction (DR) is an essential tool for removing noise\cite{3-3}, reducing redundancy\cite{3-4} and so extracting appropriate new features\cite{3-5}. 
Singular value decomposition is one of the main matrix decomposition methods that could be used as a DR method and is related to the well-known PCA method. Since in this method each data is considered as a vector, so the drawbacks of the vector-based methods in dealing with data like matrices and tensors exist for SVD, too. Generalized low-rank approximation of matrices (GLRAM) is an extension of SVD method for data samples in the matrix format, which by one pair of left and right projectors, transfers the data into a smaller subspace without folding the data into the vector format\cite{1-8}. ALso, in some data with small number of sample this methods works like SVD method even with smaller space.


In recent years some variants of this method have been proposed. For example, since at each iteration of GLRAM  two SVD should be computed, this increases the time complexity. The authors in\cite{2-7} show that instead of SVD, its approximation by Lanczos could be used which improves its speed.

Although GLRAM  preserve the spatial relationships of features and has less complexity than SVD, its search space is smaller than SVD\cite{1-8}. 
In this paper, to overcome this drawback of GLRAM method,  we proposed a method that by applying k-pair of left and right projectors to data,
while maintaining good properties of the GLRAM method has a larger search space.  This new method will be named Multiple-Paris of GLRAM (GLRAM).
Expanding th search space of multilinear method at the first is done for STM method by 
Hou et al, in their paper\cite{1-7}. They proposed a multiple-rank multilinear SVM  for classification, that expands the search space of STM in order to gain a more accurate answer same as SVM.
Theoretically, we show that by this multi-pair projections the search space of the obtained method is increased. So the quality of approximations in this method will be better than GLRAM.  Experiments show the quality of the proposed method.

In machine learning, there is a trade-off between the number of variables and occurrence of the overfitting. Although by increasing the number of parameters the search space of the model increases but at the same time this increase the occurrence of over-fitting.
In experiments, we found that, although the search space of our proposed method becomes a  bit
larger than GLRAM,   its quality always becomes better than GLRAM.  Also, despite its low search space in comparison with SVD, almost gives better or equal results in comparison with SVD. This could be interpreted by the overfitting phenomenon.  It's clear that for SVD with larger search space, the possibility of occurrence of the overfitting is more than
our proposed method, especially for data with the larger "$\#$feature-$\#$sample" ratio.
Also, it should be mentioned that the same idea could be applied to other tensor-based dimension reduction methods like MPCA and MLDA, easily.

The rest of this paper is organized as follows. In section 2, we analyze SVD and GLRAM methods and the relationship between them.
In section 3, we present our proposed method.
Next, the experimental results will be discussed in section 4. And finally, the conclusion stated in section 5. 
\section{Related works}

In real applications, data usually contains some noisy and redundant features which affect the quality of the learning process, especially for high dimensional data. Dimensionality Reduction (DR) is a process that by transforming data into a lower dimension, tries to eliminate noise and redundancy in data\cite{1-6}\cite{1-9}. 
Therefore, the occurrence of the curse of dimensionality and other undesired properties of high-dimensional spaces will be reduced, which has an important role in many applications\cite{2-0}.
In the last decades, a large number of DR techniques with different viewpoints like PCA, SVD, Fisher LDA and so on, have been investigated. Similar to other ML methods the input of the mentioned DR methods should be in vector format and so data types like images and videos (Matrix or Tensor) should be represented as a vector. This folding of high-order data to vector not only has a high complexity but also can cause losing some important spatial relations of the features in the data.
In recent years some multilinear versions of the mentioned DR methods have been proposed which are able to work with high-order data like matrices and tensors directly without reshaping them to the vector format. For example, MPCA, GLRAM, and Multilinear LDA are the multilinear versions of PCA, SVD, and LDA, respectively.

The main advantage of the multilinear methods could be summarized as follows:
\begin{itemize}
	\item They maintain the structure and so the spatial relations in the data.
	\item The parameters of the multilinear methods are less than the vector-based ones and so the computational complexity becomes less than linear methods. Also, for multilinear methods, the probability of occurrence of overfitting will be decreased  
\end{itemize} 
However, the multilinear methods are not convex typically, and also their search space is much smaller than linear ones. 
In this paper, we expand the search space of GLRAM, which can cause to gain better results. It should be mentioned that this approach could be applied to other multilinear DR methods, too. In the following we review some linear and multilinear DR methods.
\subsection{Linear DR methods based on low-rank approximation}

PCA and SVD are the main DR methods that are related to each other. Let $X=[x_1,..., x_N]\in \mathbb{R}^{n\times N}$ be a centralized data set. The principal component analysis (PCA) project the data from $\mathbb{R}^n$ to $\mathbb{R}^d, (d\ll n)$, by orthogonal transformation $W\in \mathbb{R}^{n\times d}$ such that the variance of the projected data $Y=W^TX$ maximized.
It is easy to show that this can be formulated as follows:
\begin{align}\label{3-0}
&\max_W \	 trace(W^TXX^TW)\\ \nonumber
&s.t. \	\	W^TW=I.
\end{align}
The $k$ eigenvectors corresponding to the $k$ large eigenvalues of $XX^T$ are the columns of the solution $W$ in Eq.\ref{3-0}. 
 Hence, if $X=U\Sigma V^T$ be the Singular Value Decomposition (SVD) of $X$, 
 the first $d$ left singular vectors $U_k=[u_1,..., u_d]$ are the $d$ first eigenvectors of $XX^T$ and so $W=U_d$ \cite{2-1}. Therefore $Y_d=U_d^TX$ becomes the projection of $X$. By SVD, it is clear that the projected data $Y_d$ becomes 

\begin{equation}
Y_d=U_d^TX= \Sigma_d V_d^T,
\end{equation}
where $\Sigma_d= diag(\sigma_1, \sigma_2,...,\sigma_d)$.\\
In addition, this projection could be interpreted with another viewpoint. Since $X_k=U_d\Sigma_d V_{d}^{T}=WY_d$, is the best rank-k approximation of $X$, so we found that $Y_d$ is a reduced form of original data $X$ such that have the gives the smallest construction error and, the PCA equals to the following problem
\begin{align}
&\min_{Y_d, W} \Vert X-WY_d\Vert_F^2\\ \nonumber
&s.t. \	\	W^TW=I.
\end{align}
Therefore, the PCA and SVD dimension reduction also could be rewritten as follows:
\begin{align}\label{3-1}
&\min_{W,Y} \sum _{i=1}^N \Vert x_i- Wy_i\Vert_F^2\\ \nonumber
&s.t. \	\	W^TW=I.
\end{align}
It should be mentioned that for general data matrix $X$, where is not centralized, the SVD on $\bar{X}=[x_1-\mu, x_2-\mu, ..., x_N-\mu]$ where $\mu$ is the mean of the data is equal to applying PCA on $X$ .

\subsection{Generalized low-rank approximations}
Nowadays by increasing the usage of matrix datasets, the SVD (PCA) could not be used directly on data. Now, If we have a dataset like $\left\lbrace A_1,..., A_N\right\rbrace$ where $A_i \in \mathbb{R}^{n_1\times n_2}$, to apply SVD (PCA), every matrix $A_i$ should be fold to a vector as follows:
\begin{equation}
a_i= Vec(A_i)=\left[ a_{1}^{i^T},a_{2}^{i^T},...,a_{n_2}^{i^T}\right] ^T,
\end{equation}
where $a_j^i$ is the $j$-th column of matrix $A_i$.
It is obvious that this folding maybe destroys some spatial relations\cite{0-9}. Figure.\ref{fig.1} shows this phenomenon.
\begin{figure}[h]
	\centering
	\includegraphics[scale=.25]{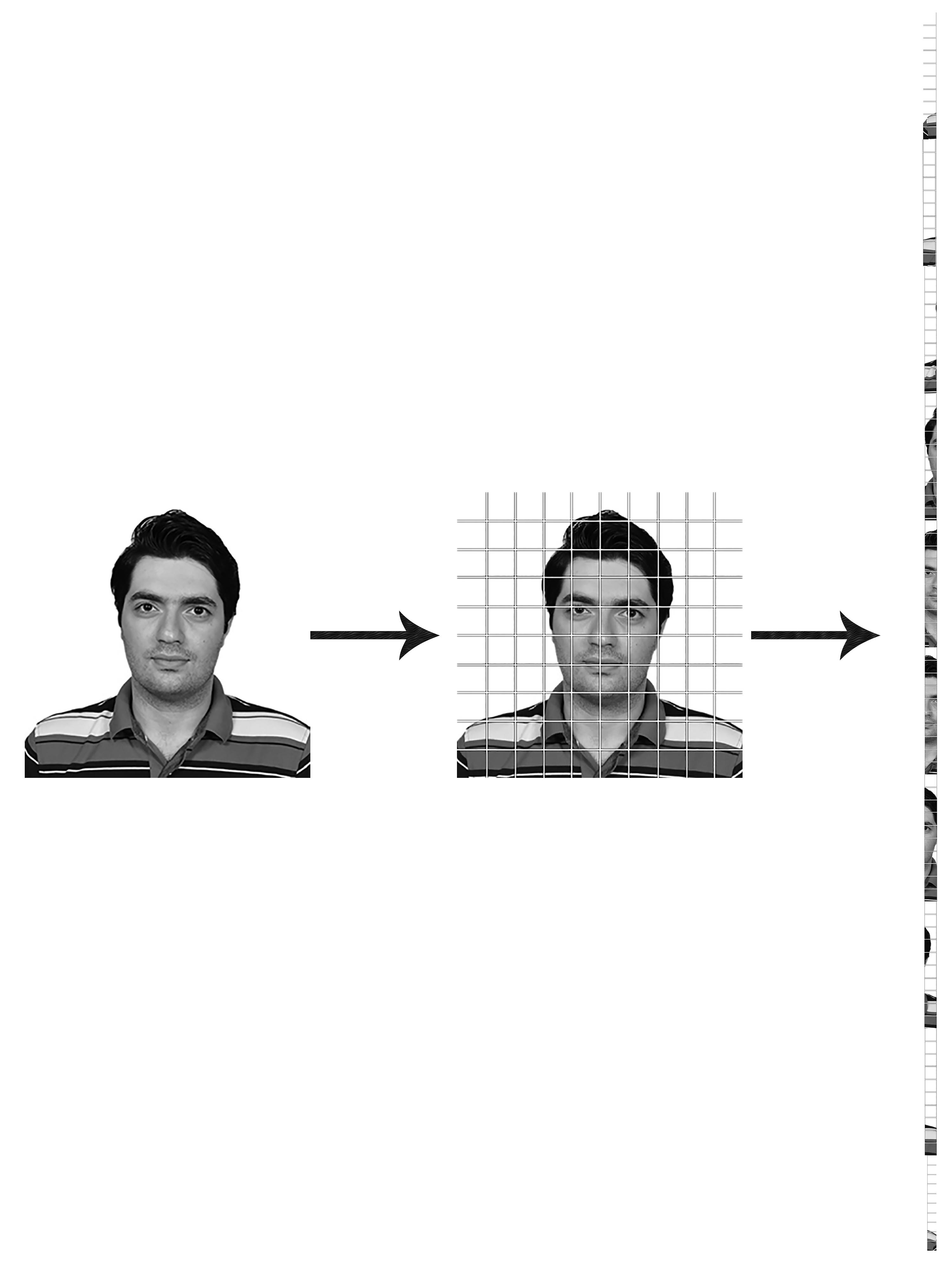}
	\caption{{\footnotesize{ Vectorizing a grayscale image
	}}}\label{fig.1}
\end{figure}

Recently, an extension of the DR based on a low-rank approximation to matrix data named Generalized Low-rank Approximation of Matrices(GLRAM) is investigated. In dimension reduction on $A_i$,
GLRAM by unknown orthogonal transformation matrices $L\in \mathbb{R}^{n_1\times k_1}$ and $R\in \mathbb{R}^{n_2\times k_2}$ looks for reduced data $D_i\in \mathbb{R}^{k_1\times k_2}$ where its reconstruction $LD_iR^T$ be the best low-rank approximation of $A_i$. Mathematically this can be modeled as follows
\begin{align} \label{1-3}
‎&\min_{\substack{L\in \mathbb{R}^{n_1\times k_1}‎ :‎\	L^TL=I_{k_1}\\ R\in \mathbb{R}^{n_2\times k_2}‎ :\	 ‎R^TR=I_{k_2}\\ D_i\in \mathbb{R}^{k_1\times k_2}‎ : ‎i=1,2,...,n}} \	\sum_{i=1}^N \Vert A_i-LD_iR^T\Vert	_F^2. 
‎\end{align}‎

Jieping Ye in his article\cite{1-8} showed that the optimal values of $L$ and $R$ should be the solution of the following maximization problem
\begin{align}\label{2-1}
&\max_{\substack{L\in \mathbb{R}^{n_1\times k_1} :\	L^TL=I_{k_1}\\ R\in \mathbb{R}^{n_2\times k_2} :\	 R^TR=I_{k_2}}} \	\sum_{i=1}^N \Vert L^TA_iR\Vert	_F^2,
\end{align}
and the optimal value of $D_i$ is $D_i=L^TA_iR$.
So instead of solving Eq.\ref{1-3}, tried to solve Eq.\ref{2-1}.
Also, to solve Eq.\ref{2-1} an alternating schema is used and at each step, this equation is substituted with the following two subproblems according to $R$ and $L$
\begin{align}\label{3-5}
&\max_R \	trace(R^TM_RR)\nonumber \\
&s.t\ \ R^TR=I,
\end{align}
and
\begin{align}\label{3-6}
&\max_L \	trace(L^TM_LL)\nonumber \\
&s.t\ \ L^TL=I,
\end{align}
where $M_R=\sum_{i=1}^{N} A_i^TLL^TA_i$ and $M_L=\sum_{i=1}^{N} A_iRR^TA_i^T$.
The optimal value of $L\in \mathbb{R}^{n_1 \times k_1}$ in Eq.\ref{3-5} and $R\in \mathbb{R}^{n_2\times k_2}$ in Eq.\ref{3-6} are the first $k_1$ and $k_2$ eigenvectors of $M_L$ and $M_R$ matrices, respectively\cite{1-8}.
The process of solving the GLRAM problem has been shown in Algorithm.\ref{alg:1.1} in detail.
\begin{algorithm}[!h]
	\caption{GLRAM}\label{alg:1.1}
	\begin{algorithmic}[1]
		\REQUIRE Input: Matrices $\lbrace A_i\rbrace _{i=1}^N$ \\
		\ENSURE Output: matrices L,R, and $\lbrace D_i\rbrace _{i=1}^N$
		\STATE Obtain initial $L_0$ and set $\textit{i}\leftarrow 1$
		\STATE While not convergent 
		\STATE \ \ \ \ Form the matrix $M_R=\sum_{j=1}^{n} A_j^TL_{i-1}L_{i-1}^TA_j$
		\STATE \ \ \ \ Compute the $k_2$ eigenvectors $\{\phi_j^R\}, j=1,\ldots,k_2$ of $M_R$ \\
		\ \ \ \ corresponding to the largest $k_2$ eigenvalues 
		\STATE \ \ \ \ $R_i \leftarrow[\phi _1^R,...,\phi _{k_2}^R]$	
		\STATE \ \ \ \ Form the matrix $M_L=\sum_{j=1}^{n} A_jR_iR_i^TA_j^T$
		\STATE \ \ \ \ Compute the $k_1$ eigenvectors $\{\phi _j^L\},{j=1, \ldots,k_1}$ of $M_L$ \\
		\ \ \ \ corresponding to the largest $k_1$ eigenvalues
		\STATE \ \ \ \ $L_i \leftarrow[\phi _1^L,...,\phi _{k_1}^L]$
		\STATE \ \ \ \ $\textit{i} \leftarrow \textit{i}+1$	
		\STATE EndWhile
		\STATE $L\leftarrow L_{i-1}$
		\STATE $R\leftarrow R_{i-1}$
		\STATE For j from 1 to n 
		\STATE     \ \ \ \ $D_j \leftarrow L^TA_jR$
		\STATE EndFor
	\end{algorithmic}
\end{algorithm}
Due to heavy computation for obtaining each eigenvalue, particularly in a deal with large data, the time and computational complexity will be increased which cause dire problems. Then, a variant of this method has been proposed named Bilinear Lanczos components (BLC)\cite{2-7} by using Lanczos method, operate faster than getting eigenvalues exactly.

\section{Low rank Approximation of matrices based on multiple-pairs of left and right transformations on matrix Samples}
In this section, we investigate the benefits and drawbacks of GLRAM method over SVD and based on this investigation propose a method that by preserving the benefits of GLRAM method try to cover its drawbacks. To start we analyze the relation between GLRAM and SVD.

Consider the data becomes $\lbrace A_1,...,A_N\rbrace$, where each  $A_i\in \mathbb{R}^{n_1\times n_2}$. For reduction of each sample to data $D_i\in \mathbb{R}^{k_1\times k_2}$, the objective  function of GLRAM method is Eq.\ref{1-3} with variables $L\in \mathbb{R}^{n_1\times k_1},R \in \mathbb{R}^{n_2\times k_2}$. But SVD works on vectorization $a_i=vec(A_i)\in \mathbb{R}^n$ of each sample. Here, the objective function of SVD (best low-rank approximation) on these data becomes
\begin{align} \label{3-3}
‎&\min_W \	\sum_{i=1}^N \Vert a_i-Wy_i\Vert	_F^2, \   \   \   \   \  \  n=n_1n_2, \ d=k_1k_2, \\‎
‎&W\in \mathbb{R}^{n\times d}, \ W^TW=I.‎\nonumber\\‎
‎&y_i\in \mathbb{R}^{d},‎\nonumber 
‎\end{align}‎
By using the properties of Kronecker product\cite{2-4} and 
vectorization, it is easy to show that GLRAM model in  Eq.\ref{1-3},  is equal to the following form
\begin{align} \label{3-4}
‎&\min_{L,R,D_{i}} \	\sum_{i=1}^N \Vert a_i-(L\otimes R)d_i\Vert	_F^2.\\‎
‎&L\in \mathbb{R}^{n_1\times k_1}‎ :‎\	L^TL=I_{k_1}\nonumber\\‎
‎&R\in \mathbb{R}^{n_2\times k_2}‎ :\	 ‎R^TR=I_{k_2}.\nonumber 
‎\end{align}‎
Now, by comparison between the Eq.\ref{3-3} and Eq.\ref{3-4}, The following benefits of GLRAM over SVD could be understood:
\begin{itemize}
\item	GLRAM works directly on data with their own format without folding them into vectors.
\item In GLRAM we are free to choose the amount reduction in each arbitrary mode,which is not meaningful in SVD. 
	\item GLRAM has $n_1k_1+n_2k_2$ parameters that should be estimated, while there are $n_1n_2k_1k_2$ ones for SVD. So  the complexity  and so the  possibility of overfitting in GLRAM is much less than SVD.	
\end{itemize}
Although, fewer parameters in GLRAM has the mentioned benefits, we show that this causes GLRAM has the smaller search space over the SVD.
To show this fact mathematically, GLRAM objective function has could be rewritten as follows
\begin{align}\label{3-7}
\sum_{i=1}^{N}\Vert A_i - LD_iR^T\Vert_F^2
=\sum_{i=1}^{N}\Vert a_i-(L\otimes R)d_i\Vert_2^2. 
\end{align}
So,
\begin{align}\label{4-3}
\Phi =\left\lbrace V \vert V=L\otimes R,  L\in \mathbb{R}^{n_1\times k_1} \ and\  R\in \mathbb{R}^{n_2\times k_2}\ are\  orthogonal \right\rbrace,
\end{align}
 denotes the search space of GLRAM. But the search space of SVD method is
\begin{align}\label{4-33}
\Psi =\left\lbrace U \vert  U \in \mathbb{R}^{n\times d} \ is \ an\ orthogonal\ matrix\,\  n=n_1n_2, d=k_1k_2\right\rbrace .
\end{align}
Therefore, it is clear that $\Phi \subset \Psi$ and the search space of GLRAM is a subset of the search space of the SVD method. 
We can see the summary of comparisons between vector-based methods and tensor-based ones in Table.\ref{tab.1}.
\begin{small}
	\begin{table}[h]
		\centering
		\caption{\footnotesize{
				Comparison between vector-based and tensor-based  methods
		}}\label{tab.1}
		\begin{tabular}{| c || c | c |}
			\hline
			Methods & Vector-Based & Tensor-Based \\
			\hline
			\hline
			Search space & Large & Small  \\
			\hline
			Complexity & High & Low \\
			\hline
			Spatial Relationship & Ignored & Considered\\
			\hline
		\end{tabular}
	\end{table}
\end{small}
\subsection{Proposed method by multi-pair of projections}
In this section, we try to extend the search domain of GLRAM in order to improve its quality without losing its aforesaid advantages.
To design our proposed method, we should have new insight to search region $\Psi$ of SVD method.

\begin{lemma}
Consider  $W\in \mathbb{R}^{n\times k}$ be a solution of SVD  model in Eq.\ref{3-3} applied on vectorization $\left\{a_i=vec(A_i)\right\}$ of data samples $\{A_i\in \mathbb{R}^{n_1\times n_2}\}, i=1,\ldots,N$ and $n=n_1n_2$ to reduce them to vectors $\{y_i\in \mathbb{R}^{d}\}$, where $d=k_1k_2$. Depended to the data  there exists an integer number $l \leq \min\{n_1k_1,n_2k_2\}$ such that, $W$ could be rewritten as the following form
\[
W=\sum_{j=1}^{l}L_j \otimes R_j
\] 
where  $L_j\in \mathbb{R}^{n_1\times k_1}, R_j \in \mathbb{R}^{n_2\times k_2} $ for $ j=1,\ldots, l$ .
\end{lemma}
\begin{proof}
When SVD model Eq.\ref{3-3} is applied on data samples $\{A_i\in \mathbb{R}^{n_1\times n_2}\}$, according to Eq.\ref{4-33}, the solution will be lie on the feasible set $\Psi$ and each feasible solution will be an orthogonal matrix $W \in \mathbb{R}^{n\times d}$, where $n=n_1n_2$ and $d=k_1k_2$. Now we design the following partitioning on 
matrix $W$
\begin{align}
&W=
\left(
\begin{array}{c|c|c}
W _{11} &\cdots   & W _{1k_1}\\ 
\hline
\vdots &  \ddots & \vdots \\ 
\hline
W _{n_11} &\cdots   &  W _{n_1k_1}
\end{array} \right),
\end{align}
where $W$ contains $n_1k_1$ numbers of block matrices $W_{ij}\in \mathbb{R}^{n_2\times k_2}, i=1,..,n_1, \	j=1,...,k_1$. Based on this partitioning we define the follwing reshaping
\cite{3-2}
\begin{equation}\label{4-2}
\tilde{W}=[Vec(W _{1,1}),..., Vec(W _{n_1,1}),..., Vec(W _{1,k_1}),..., Vec(W _{n_1,k_1})]^T\in \mathbb{R}^{n_1k_1\times n_2k_2}.
\end{equation}
If $rank(\tilde{W})=l$ where $l \leq \min \{n_1k_1,n_2k_2\}$, 
the SVD decomposition of $\tilde{W}$ will be
\begin{equation}\label{3-8}
\tilde{W}=\sum_{i=1}^l \sigma_i u_iv _i^T.
\end{equation}
By defining $\bar{u}_i=\sqrt{\sigma_i}u _i  \in \mathbb{R}^{n_1k_1}$ and $\bar{v}_i=\sqrt{\sigma_i}v_i  \in \mathbb{R}^{n_2k_2}$ the Eq.\ref{3-8} becomes
\begin{equation}
\tilde{W}=\sum_{i=1}^l \bar{u}_i\bar{v}_i^T.
\end{equation}
By the properties of Kronecker product and definition of reshaped matrix $\tilde{W}$, it is easy to show that
\begin{equation}\label{3-9}
W=\sum_{j=1}^l L_j\otimes R_j,
\end{equation}
where $Vec(L_j)=\bar{u}_j$ and $Vec(R_j)=\bar{v}_j$\cite{2-4, 3-2}.
\end{proof}
 This shows that every projection matrix $W \in \Psi$ has a form like Eq.\ref{3-9} and so the  projection matrix of GLRAM belongs to $\Phi$ as in Eq.\ref{4-3} is a special case of Eq.\ref{3-9} when $l=1$.\\
By this relation, if we set a $1<k<l$, using the projection matrix like 
\begin{equation}\label{4-0}
W =\sum_{j=1}^k L_j\otimes R_j,
\end{equation}
enables us to use the benefits of GLRAM and SVD at the same time. This means that by the mentioned $W$ in Eq.\ref{4-0} as a projection matrix in GLRAM model, we obtained the following model which will be named Multiple-pairs of GLRAM (MPGLRAM).
\begin{align}\label{4-1}
&\min_{\{L_j, R_j\}_{j=1}^{k},
\{	D_i\}_{i=1}^{N}} \sum_{i=1}^N \Vert a_i-\sum_{j=1}^k (L_j\otimes R_j)d_i\Vert_F^2\\ \nonumber
&=\sum_{i=1}^N\Vert A_i-\sum_{j=1}^k L_jD_iR_j^T\Vert_F^2.
\end{align}
Here $d_i=Vec(D_i)$ and $D_i\in \mathbb{R}^{k_1\times k_2}$ which similar to GLRAM works on matrix data with their own format and at the same time its search space is larger than GLRAM method. 
Here, 
\[\overline{\Phi}_k=\left\lbrace \sum_{j=1}^k L_j\otimes R_j \vert L_j \in \mathbb{R}^{n_1\times k_1}, R_j\in \mathbb{R}^{n_2\times k_2} \right\rbrace \]
denotes the search space of MPGLRAM model in Eq.\ref{4-1} and from Eq.\ref{4-3} we can conclude that $\Phi \subset \overline{\Phi}_k$. 
In the following we list the appropriate properties of the proposed model:
\begin{itemize}
\item The search space of the proposed method is larger than GLRAM method.
\item From Eq.\ref{4-1} it is clear that our proposed method is applied to the data with their own format without folding to vectors.
\item This proposed method has $(n_1k_1+n_2k_2)k$ parameters that should be estimated. Since we consider $k$ as a small number, the complexity of this method is not much higher than GLRAM, and still, the probability of occurrence of overfitting is less for this method in comparison with SVD or PCA.
\end{itemize}
\subsection{Solving the proposed model}
In the proposed MPGLRAM model, we deal with the following minimization problem:
\begin{align}‎\label{1-4}
&\min_{\substack{L_j\in \mathbb{R}^{n_1\times k_1}‎ :j‎=1,2,...,k\\ R_j\in \mathbb{R}^{n_2\times k_2}‎ j=1,2,...,k \\ D_i\in \mathbb{R}^{k_1\times k_2} i=1,2,...N,}} \	\sum_{i=1}^N \Vert A_i-\sum_{j=1}^k L_jD_iR_j^T\Vert	_F^2.
‎\end{align}‎

To solve Eq.\ref{1-4} like GLRAM, we use a coordinate descent\cite{3-1} approach. So at each step of the algorithm, we have some subproblems that are solved only according to one variable. 
So, after $p$ steps let 
$L_j^{(p)}$, $R_j^{(p)}$ and ${D_i^{(p)}}$ are the estimations of projections and data matrices.
At the first in this step  we consider the matrices $L_j^{(p)}$ and $R_j^{(p)}$ be known from last step and try to update reduced data $\{D_i\}, i=1,\ldots,N$. So, this leads to the following subproblem
\begin{align}\label{1-5}
\lbrace D_i^{(p+1)}\rbrace_{i=1}^N&=& \arg\min _{\{D_i\}_{i=1}^{N}} \sum_{i=1}^N\Vert A_i-\sum_{j=1}^k L_j^{(p)}D_iR_j^{(p)^T}\Vert_F^2\nonumber\\
& =&\arg\min _{\{d_i\}_{i=1}^{N}} \sum_{i=1}^{N}\Vert a_i-(\sum _{j=1}^k R_j^{(p)} \otimes L_j^{(p)})d_i\Vert_2^2,
\end{align}
where $d_i={\sf vec}(D_i), \quad a_i={\sf vec}(A_i) $.
If we set $B^{(p)}=\sum_{j=1}^{k}\left( R_j^{(p)}\otimes L_j^{(p)}\right)$, this problem can be reformulated as the following least squares problem \cite{2-5}.
\begin{align}\label{1-6}
&\min _{\{d_i\}_{i=1}^{N}}\sum_{i=1}^N \Vert a_i-B^{(p)}d_i\Vert _2^2 
=\min_{D}  \Vert A-B^{(p)}D\Vert _F^2,,
\end{align}
where $A=[a_1,...,a_N]\in \mathbb{R}^{n_1n_2\times N}$ and $D=[d_1,...,d_N]\in \mathbb{R}^{k_1k_2\times N}$. 
This is a well-known least square problem and could be solved easily by direct and iterative matrix computation techniques.

After solving the mentioned problem we should find $L_j$, $R_j$ parameters successively by coordinate descent approach for j=1,...,k.\cite{3-1}
So if we assume 
$\lbrace L_j, R_j\rbrace_{  j=1,...,j'-1,j'+1,...,k}$ and $\{D_i\}_{i=1}^{N}$ are known, we should estimate $L_{j'}$ and $R_{j'}$ in the next step. By these assumption equation Eq.\ref{1-4}, according to Eq.\ref{1-5} and Eq.\ref{1-6} leads to
\begin{align}‎\label{1-7}
‎&\min_{\substack{L_{j'}\in \mathbb{R}^{n_1\times k_1} \\ R_{j'}\in \mathbb{R}^{n_2\times k_2}‎  ‎}} \	\sum_{i=1}^N  \| A_i-\sum\limits_{\scriptstyle j = 1\hfill\atop
	\scriptstyle j \ne j'\hfill}^k L_jD_iR_j^T - L_{j'}D_iR_{j'}^T  \|	_F^2.
\end{align}‎
By replacing $\bar{A}_i= A_i-\sum\limits_{\scriptstyle j = 1\hfill\atop
	\scriptstyle j \ne j'\hfill}^k L_jD_iR_j^T$ this equation becomes
\begin{equation} \label{1-8}
\min_{\substack{L_{j'}\in \mathbb{R}^{n_1\times k_1}‎ \\ R_{j'}\in \mathbb{R}^{n_2\times k_2} }} \	\sum_{i=1}^N \Vert \bar{A_i}- L_{j'}D_iR_{j'}^T \Vert	_F^2.\\
\end{equation}

For solving Eq.\ref{1-8} we use an alternating schema. At first, we fixed $L_{j'}$
and solve the problem according to $R_{j'}$.
By replacing $M_i=L_{j'}D_i$ in Eq.\ref{1-8} and  regarding to  properties of Trace function of matrices we have \cite{2-6}
\begin{align}
&\Vert \overline{A_i}- M_iR_{j'}^T \Vert	_F^2=
tr\left((\overline{A_i}- M_iR_{j'}^T)^T(\overline{A_i}- M_iR_{j'}^T)\right)\nonumber\\
&=tr(\overline{A_i}^T\overline{A_i}-2\overline{A_i}^TM_iR_{j'}^T+R_{j'}M_i^TM_iR_{j'}^T).
\end{align}
By removing the constants in this term the Eq.\ref{1-8} leads to the following problem
\begin{align}
&\min_{R_{j'}} \ 	\sum_{i=1}^Ntr(-2\overline{A_i}^TM_iR_{j'}^T+R_{j'}M_i^TM_iR_{j'}^T)\nonumber\\
&=\min _{R_{j'}} \ 	-2\sum_{i=1}^Ntr(\overline{A_i}^TM_iR_{j'}^T)+\sum_{i=1}^Ntr(R_{j'}M_i^TM_iR_{j'}^T) \nonumber\\
&=\min _{R_{j'}} \ 	-2tr((\sum_{i=1}^N\overline{A_i}^TM_i)R_{j'}^T)+tr(R_{j'}(\sum_{i=1}^NM_i^TM_i)R_{j'}^T),
\end{align}
By defining $N_R=\sum_{i=1}^{N}\bar{A}_i^TM_i$ and $B_R=\sum_{i=1}^{N}M_i^TM_i$, 
this optimization problem becomes
\begin{align} \label{1-9}
&\min _{R_{j'}}\  -2tr(N_RR_{j'}^T)+tr(R_{j'}B_RR_{j'}^T),
&R_{j'}\in \mathbb{R}^{n_2\times k_2}.
\end{align}
 This problem is quadratic convex and so its derivative according to $R_{j'}$ in the optimal point should be zero. Therefore 
 by setting the derivative of  the objective function equal to zero, we have
\begin{equation}
-2N_R+2R_{j'}B_R=0,
\end{equation}
and consequently, $R_j'$ will be 
\begin{equation}
R_{j'}=N_RB_R^{-1}.
\end{equation}

Also, the  same  process could be used to find $L_{j'}$. Here by known $R_{j'}$ and setting $M_i=R_{j'}D_i^T$  Eq.\ref{1-8} according to $L_{j'}$ becomes

\begin{align}\nonumber 
\min _{L_{j'}} \ \Vert \overline{A_i}^T- M_iL_{j'}^T \Vert	_F^2=&tr\left((\overline{A_i}^T- M_iL_{j'}^T)^T(\overline{A_i}^T- M_iL_{j'}^T)\right)\\ \nonumber
=&\min _{L_{j'}} tr( \overline{A_i}\overline{A_i}^T-2\overline{A_i}M_iL_{j'}^T +L_{j'}M_i^TM_iL_{j'}^T)\\
=&\min _{L_{j'}} \  
-2\sum_{i=1	}^N tr((\sum_{i=1}^N\overline{A_i}M_i)L_{j'}^T)+tr(L_{j'}(\sum_{i=1}^NM_i^TM_i)L_{j'}^T).
\end{align}
By replacing $N_L=\sum_{i=1}^{N} \bar{A}_iM_i$ and $B_L=\sum_{i=1}^{N}M_i^TM_i$, we have
\begin{equation}
\min _{L_{j'}}\  -2tr(N_LL_{j'}^T)+tr(L_{j'}B_LL_{j'}^T).
\end{equation}
which its solution is
\begin{equation}
L_{j'}=N_LB_L^{-1}.
\end{equation}
Since we should obtain all $k$ variables, we do previous stages $k$ times to determine $L_j$ and $R_j$, for all amount of $j=1,...,k$.
Eventually, at each time, $k-1$ parameters will be assumed to be fixed  except one parameter that should be estimated.
And for the next parameter, the updated form of the previous ones will be used. Besides, we can repeat this alternative process more than once. 
 The details of the proposed coordinate descent process can be seen in Algorithm.\ref{alg:1.0}. It is easy to show that this proposed method based on coordinate descent approach for our proposed MPGLRAM model is a descent algorithm and at each step the objective function decrease.
\begin{algorithm}[h]
	\caption{MPGLRAM} \label{alg:1.0}
	\begin{algorithmic}[1]
		\REQUIRE Input: matrices $\lbrace A_i\rbrace _{i=1}^N$,$k$, $iter$ \\
		\ENSURE Output: matrices $\lbrace D_i\rbrace _{i=1}^N$
		\STATE Initialize $\lbrace L_j,R_j\rbrace _{j=1}^k$	
		\STATE Construct $B=\sum_{j=1}^k (R_j\otimes L_j)$
		\STATE Update $D$ by solving Eq.\ref{1-6}
		\STATE for i from 1 to \textit{iter}
		\STATE \ \ \ \ For each $j', j'=1,...,k$ Update  ${L_j'}$ and ${R_j'}$ by ?? and ??.
		\STATE \ \ Update $D$ by solving Eq.\ref{1-6}
		\STATE EndFor
	\end{algorithmic}
\end{algorithm}

\begin{lemma}
The proposed coordinate descent algorithm for MPGLRAM is a decreasing process.
\end{lemma}
\begin{proof}
The objective function of MPGLRAM method is
\begin{equation}
	\min_{\{L_{j},R{j}\}_{j=1}^{k},\{D_{i}\}_{i=1}^{N}} f(\{D_i\}_{i=1}^N, \{L_j,R_j,\}_{j=1}^k),
\end{equation}
where 
\[
f\left(\{D_i\}_{i=1}^N, \{L_j,R_j\}_{j=1}^k\right)=\sum_{i=1}^N \Vert A_i-\sum_{j=1}^k L_jD_iR_j^T\Vert	_F^2.
\]
For simplicity, we set $Z_1=\{D_1,...,D_N\}, Z_{2l}=L_l,$ and $Z_{2l+1}=R_l$ for $l=1,...,k$. So, we have
\begin{equation}\label{n1}
	g(Z_1,...,Z_{2k+1})=f(\{D_i\}_{i=1}^N, \{L_j,R_j,\}_{j=1}^k),
\end{equation}
and to show the decreasing property of the proposed algorithm, it is enough to indicate
\begin{equation}
	g(Z_1^{(p+1)},...,Z_{2k+1}^{(p+1)})\leq g(Z_1^{(p)},...,Z_{2k+1}^{(p)})
\end{equation}
where $Z_i^{(p+1)}$,$Z_i^{(p)}$ for $i=1,...,2k+1$ denote the approximations of solution of Eq.\ref{n1} at  $(p+1)$-th and $p$-th steps, respectively. Moreover, it is obvious that at $(p+1)$-th step of the proposed algorithm, coordinate descent process is applied $(2k+1)$ times, and at each step, one of the variables is updated. As a result, we can maintain that
\begin{align}\nonumber
	g(Z_1^{(p+1)},...,Z_{2k+1}^{(p+1)})=&\min_{Z_{2k+1}} 	g(Z_1^{(p+1)},...,Z_{2k}^{(p+1)},Z_{2k+1})\\ \nonumber
	&\leq g(Z_1^{(p+1)},...,Z_{2k}^{(p+1)},Z_{2k+1}^{(p)})\\ \nonumber
	&\vdots \\ \nonumber
	&\leq \min_{Z_l} g(Z_1^{(p+1)},...,Z_{l-1}^{(p+1)},Z_l,Z_{l+1}^{(p)},Z_{2k+1}^{(p)})\\ \nonumber
	&\leq g(Z_1^{(p+1)},...,Z_{l-1}^{(p+1)},Z_l^{(p)},Z_{2k+1}^{(p)})\\\nonumber
	&\vdots\\ \nonumber
	&\leq g(Z_1^{(p)},...,Z_{2k+1}^{(p)}),
	\end{align}
	which finish the proof.
\end{proof}
This proves the decreasing behavior of the proposed algorithm to solve MPGLRAM.
The number of parameters that should be estimated in MPGLRAM method by k-pair of projections is $k(m+n)$, which is much less than $mn$ parameters in SVD. But still 
close to the $(m+n)$ parameters of GLRAM method for small $k$.
 Therefore, by using k pairs of projectors we expand the search space merely to find the optimal answer and protect it from tending to overfitting due to many parameters like SVD.   

\section{Experimental Results}

In this section to show the quality of the proposed MPGLRAM method, we present some experiments on  well-known data sets 
ORL
\footnote{\text{http://www.cad.zju.edu.cn/home/dengcai/Data/FaceData.html}}
, Yale
\footnote{\text{http://web.mit.edu/emeyers/www/face\underline{ }databases.html}}
, YaleB
\footnote{\text{http://vision.ucsd.edu/~leekc/ExtYaleDatabase/ExtYaleB.html}}
,and PIE
\footnote{\text{http://featureselection.asu.edu/old/datasets/pixraw10P.mat}}. 
The details of the data are listed in Table.\ref{tab.2}.
\begin{small}
	\begin{table}[h]
		\centering
		\caption{\footnotesize{
				Characters of different data sets.
		}}\label{tab.2}
		\begin{tabular}{ c | c | c | c }
			\hline
			Data & Size & Scale & Class number \\
			\hline
			ORL & 400 & $32\times 32$ & 40\\
			\hline
			Yale & 165 & $32\times 32$ & 11 \\
			\hline
			YaleB & 2414 & $32\times 32$ & 38\\
			\hline
			PIE & 210 & $44\times 44$ & 10\\
			\hline
		\end{tabular}
	\end{table}
\end{small}

Also Fig.\ref{fig.2} shows some samples of YaleB data set.
\begin{figure}
	\begin{center}
		\includegraphics[scale=0.20]{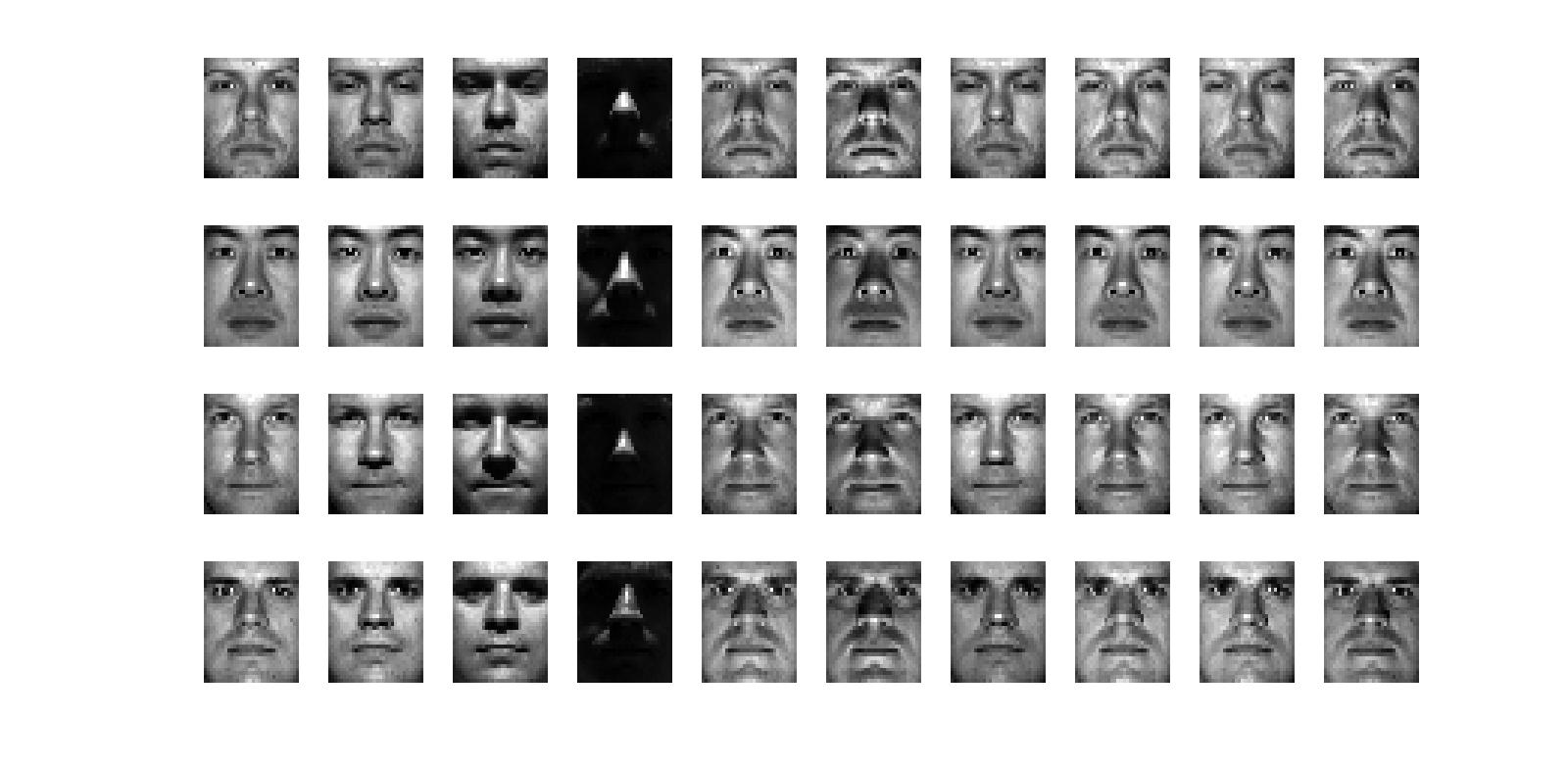}
		\caption{{\footnotesize{10 samples of 4 individuals with YaleB database}}}\label{fig.2}
	\end{center}	
\end{figure}
Here we compare  the MPGLRAM with GLRAM and SVD based  based on  the quality of reconstruction error and the accuracy of classification  on the projected data.
\subsection{Comparison based reconstruction error}
At the first type of evaluation, we applied  MPGLRAM, GLRAM and SVD methods on the mentioned data sets to project them to different smaller dimensions and used the quality of reconstructions by these projected data as evaluation of the quality of these DR methods. The quality of the reconstruction is evaluated via Root Mean Square Reconstruction Error (RMSRE) measure.

For different values of $d$ we reduced each data matrix (sample) by GLRAM and MPGLRAM
to matrices in $\mathbb{R}^{d\times d}$. Also to compare these methods with SVD we have to project the data to a vector in $\mathbb{R}^{d^2}$ by SVD method. The RMSRE of all the mentioned approximations for different values of $d$ are presented in Figure.\ref{fig.3}.
Here we applied MPGLRAM with different $k=2,\ldots,5$
\begin{figure}[!h]
	\begin{center}
		\centering
		\includegraphics[scale=0.053]{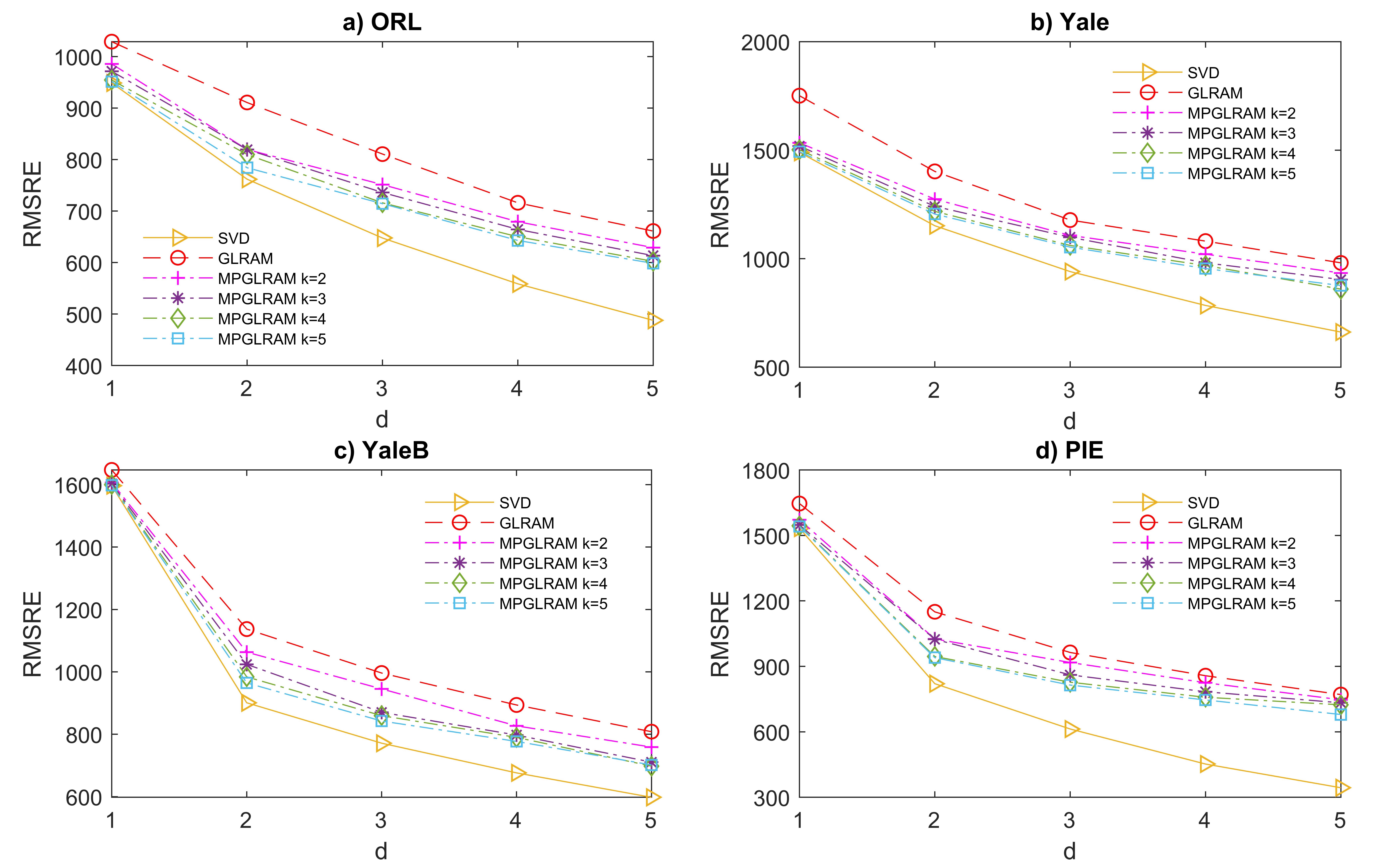}
		\caption{{\footnotesize{ RMSRE   of MPGLRAM , GLRAM and SVD  versus  different values of $a$ applied on  a) ORL, b) Yale, c) YaleB, and d) PIE data sets.}}}\label{fig.3}
	\end{center}
\end{figure}
As we can see in Figure.\ref{fig.3}, the results of the proposed MPGLRAM method by all $k$ and for all datasets and different values of $d$ are better than GLRAM. Also by increasing the value of $k$, the RMSRE of MPGLRAM is improved for all data and become near the SVD method. This confirms the effect of increasing the search space in the quality of the reconstruction of the methods.
Also, here the SVD method gives the best results. This phenomenon was predictable due to the large size of its search space.
But as we mentioned before, this large space causes a large amount of parameters that could be led to over-fitting that will be investigated in the following.
\subsection{Evaluation based on classification accuracy on Projected data}
In this section, we compare the classification accuracy on the projected data by the proposed MPGLRAM, GLRAM, and SVD methods. In our experiments, we applied \textit{$K$-fold cross-validation} measure with different amount of $K$, $K=2,5,10$.
For classification, we used discriminant analysis classifier, which in Matlab its command is ${\sf fitcdiscr}$\cite{bish}. In experiments, we found that this classifier works better than KNN and SVM for our data sets.
From Table.\ref{tab.2}, its clear that the ratio of the number of samples over the number of features for ORL, Yale, PIE, and YaleB are $0.39, 0.16, 0.11$ and $2.36$, respectively.
So, we could understand that the possibility of overfitting for YaleB is much less than other data sets. This means that we expect that for this data set the SVD method could work better than others, but for other datasets, due to the possibility of overfitting we could expect that multilinear methods could work better than SVD.
In the following by the experiments we will investigate this issue. 

We start with YaleB dataset.  Here we applied SVD, GLRAM, and MPGLRAM on this data to project each sample to a vector with dimension $d^2$ for SVD and matrices with dimensions $d\times d$ for GLRAM and MPGLRAM, for different values $d=5,6,7,8,9$.
Also for MPGLRAM we report only the best results of the best rank-$k$ between 2,3,4,5.
The obtained accuracies of all mentioned methods could be seen in Table.\ref{tab.5}.
\begin{small}
	\begin{table}[H]
		\centering
		\caption{\footnotesize{
				Comparison the percentage of accuracy in SVD, GLRAM, MPGLRAM on YaleB.
		}}\label{tab.5}
		\scalebox{0.7}[0.7]{
		\begin{tabular}{|c|c|c|c|c|}
			\hline
			\textbf{k-fold} & \textbf{d} & \textbf{SVD} & \textbf{GLRAM} & \textbf{MPGLRAM}\\ \hline
			\hline	
			\multirow{5}{*}{\textbf{2}}  & 5          & 70.34             &      62.88     &  65.53           \\ \cline{2-5} 
			& 6          &     74.40    & 69.76          & 74.52            \\ \cline{2-5} 
			& 7          &     77.84    & 74.11          & 78.21            \\ \cline{2-5} 
			& 8          & 	   80.74 	& 78.67          & 81.35            \\ \cline{2-5} 
			& 9          &     83.97    & 81.23          & 82.80             \\ \hline
			\hline
			\multirow{5}{*}{\textbf{5}}  & 5          &  71.41            &    63.63      &      66.78      \\ \cline{2-5} 
			& 6          &    75.23     & 71.00          & 76.18            \\ \cline{2-5} 
			& 7          &    79.33     & 75.23       	 & 79.08            \\ \cline{2-5} 
			& 8          &    82.44     & 80.07          & 82.27            \\ \cline{2-5} 
			& 9          &    84.18     & 82.52        	 & 84.05            \\ \hline
			\hline
			\multirow{5}{*}{\textbf{10}} & 5          & 72.04             &     64.21      &  67.61          \\ \cline{2-5} 
			& 6          & 75.23 		& 71.33			 & 76.59		     \\ \cline{2-5} 
			& 7     	 & 79.78 		& 75.60          & 79.45           \\ \cline{2-5} 
			& 8          & 82.64 		& 80.41          & 82.60            \\ \cline{2-5} 
			& 9          & 85.05 		& 82.56 	 	 & 84.42              \\ \hline
		\end{tabular}}
	\end{table}
\end{small}
From this table, it is clear that the MPGLRAM method works better than GLRAM for all dimensions and for all K-fold experiments. Also, although the parameters of $MPGLRAM$
is less than SVD, but almost these methods give a similar performance. This shows the power of the proposed method because in this situation the possibility of overfitting for SVM was less than other datasets and due to large search space we expected that SVD works better than MPGLRAM method, but the results do not show this. 

As the second data set, we consider the Yale dataset. Table.\ref{tab.4} shows the obtained results for this dataset.
\begin{small}
	\begin{table}[H]
		\centering
		\caption{\footnotesize{
				Comparison the percentage of accuracy in SVD, GLRAM, MPGLRAM on Yale.
		}}\label{tab.4}
		\scalebox{0.7}[0.7]{
		\begin{tabular}{|c|c|c|c|c|}
			\hline
			\textbf{k-fold} & \textbf{d} & \textbf{SVD} & \textbf{GLRAM} & \textbf{MPGLRAM}\\ \hline
			\hline	
			\multirow{3}{*}{\textbf{2}} & 5 & 75.15 & 73.33 & 82.42
			\\ \cline{2-5}
			& 6 & 75.15 & 75.76 & 82.42 \\ \cline{2-5}
			& 7 & 69.69 & 69.09 & 78.18 \\ \cline{2-5}
			& 8 & 58.78 & 50.30 & 58.18 \\ \cline{2-5}
			& 9 & 63.03 & 56.36 & 63.03 \\ \hline
			\hline
			\multirow{3}{*}{\textbf{5}} & 5 & 81.82 & 83.03 & 84.85
			\\ \cline{2-5}
			& 6 & 83.64 & 82.42 & 87.88 \\ \cline{2-5}
			& 7 & 83.64 & 81.82 & 89.09 \\ \cline{2-5}
			& 8 & 83.64 & 84.85 & 89.70 \\ \cline{2-5}
			& 9 & 83.03 & 81.21 & 85.45 \\ \hline
			\hline
			\multirow{3}{*}{\textbf{10}} & 5 & 79.39 & 83.64 & 86.06 
			\\ \cline{2-5}
			& 6 & 85.45 & 84.85 & 87.88 \\ \cline{2-5}
			& 7 & 84.85 & 86.67 & 90.30 \\ \cline{2-5}
			& 8 & 87.27 & 87.27 & 91.52 \\ \cline{2-5}
			& 9 & 88.48 & 87.27 & 89.09 \\ \hline
		\end{tabular}}
	\end{table}
\end{small}
This table shows that our proposed method not only works better than GLRAM, but even its performance is also better than the SVD method. Here, our proposed method achieves its best result in 
$d=5$ with accuracy $82.42$, while this for SVD is $75.15$ with $d=5$ and for GLRAM is $75.76$ with $d=6$ in 2-fold. For 5-fold the best results of MPGLRAM, SVD, and GLRAM methods are $89.70, 83.64$ and $84.85$, receptively. Here we see that the performance of the proposed method at least $5\%$ larger than its nearest competitor, i.e., GLRAM.
For 10-fold we see that the best results of MPGLRAM, SVD, and GLRAM are $91.52, 88.48$ and $87.27$ which are obtained for dimensions $8,9,8$. Here we see that our proposed method with small $d=6$ gives better accuracy in comparisons with $GLRAM$ and $SVD$ with larger dimensions $d=7,8,9$. By these explanations, we could conclude that our proposed method works better than other methods in Yale dataset. 

As a third test, we consider the results on ORL dataset. The results can be found in Table.\ref{tab.3}
\begin{small}
	\begin{table}[H]
		\centering
		\caption{\footnotesize{
				Comparison the percentage of accuracy in SVD, GLRAM, MPGLRAM on ORL.
		}}\label{tab.3}
		\scalebox{0.7}[0.7]{
		\begin{tabular}{|c|c|c|c|c|}
			\hline
			\textbf{k-fold} & \textbf{d} & \textbf{SVD} & \textbf{GLRAM} & \textbf{MPGLRAM}\\ \hline
			\hline	
			\multirow{5}{*}{\textbf{2}}  & 5          &     96.25         & 96.25          & 96.25            \\ \cline{2-5} 
			& 6          &   95.50           & 96.50          & 98.00            \\ \cline{2-5} 
			& 7          &   96.25           & 97.00          & 98.00            \\ \cline{2-5} 
			& 8          &   96.75           & 98.25          & 98.25             \\ \cline{2-5} 
			& 9          &   95.00           & 97.00          & 97.75            \\ \hline
			\hline
			\multirow{5}{*}{\textbf{5}}  & 5          &  96.75            & 97.00          & 97.75            \\ \cline{2-5} 
			& 6          &   97.75           & 98.25          & 99.25            \\ \cline{2-5} 
			& 7          &   97.75           & 98.75          & 99.50            \\ \cline{2-5} 
			& 8          &   98.25           & 99.50          & 99.50            \\ \cline{2-5} 
			& 9          &   98.75           & 99.00          & 99.50            \\ \hline
			\hline
			\multirow{5}{*}{\textbf{10}} & 5          &  96.50            & 97.00          & 98.50            \\ \cline{2-5} 
			& 6          & 98.00             & 97.75          & 99.25            \\ \cline{2-5} 
			& 7          & 99.00             & 99.00          & 99.25            \\ \cline{2-5} 
			& 8          & 98.50             & 99.25          & 99.75            \\ \cline{2-5} 
			& 9          & 98.75             & 99.25          & 99.75            \\ \hline
		\end{tabular}}
	\end{table}
\end{small}
In our experiments, as we stated in Section.3, we expect that MPGLRAM gains results better than GLRAM and in some cases, even better than SVD. Eventually, not only our results are better than GLRAM, but also there are a sizable number of cases that we reach more accurate classification than SVD. Since SVD has more parameter in comparison with our method so
in overall one could see that the proposed method has better performance in comparison with SVD. Also in all situations is better than or equal to GLRAM method.

As the last experiment, we report the result for PIE dataset in Table.\ref{tab.6}.
\begin{small}
	\begin{table}[H]
		\centering
		\caption{\footnotesize{
								Comparison the percentage of accuracy in SVD, GLRAM, MPGLRAM on PIE.	}}\label{tab.6}
	\scalebox{0.7}[0.7]{
		\begin{tabular}{|c|c|c|c|c|}
			\hline
			\textbf{k-fold} & \textbf{d} & \textbf{SVD} & \textbf{GLRAM} & \textbf{MPGLRAM}\\ \hline
			\hline	
			\multirow{3}{*}{\textbf{2}} & 5 & 100 & 97.14 & 100
			\\ \cline{2-5}
			& 6 & 98.57 & 99.04 & 100 \\ \cline{2-5}
			& 7 & 99.04 & 96.67 & 100 \\ \cline{2-5}
			& 8 & 98.09 & 96.67 & 99.52 \\ \cline{2-5}
			& 9 & 92.85 & 95.23 & 99.04 \\ \hline
			\hline
			\multirow{3}{*}{\textbf{5}} & 5 & 100 & 98.57 & 100
			\\ \cline{2-5}
			& 6 & 99.04 & 99.04 & 100 \\ \cline{2-5}
			& 7 & 99.52 & 99.04 & 100 \\ \cline{2-5}
			& 8 & 99.52 & 98.57 & 100 \\ \cline{2-5}
			& 9 & 99.04 & 99.52 & 100 \\ \hline
			\hline
			\multirow{3}{*}{\textbf{10}} & 5 & 100 & 99.04 & 100 
			\\ \cline{2-5}
			& 6 & 99.52 & 99.04 & 100 \\ \cline{2-5}
			& 7 & 99.52 & 99.04 & 100 \\ \cline{2-5}
			& 8 & 99.52 & 99.04 & 100 \\ \cline{2-5}
			& 9 & 99.04 & 99.04 & 100 \\ \hline
		\end{tabular}}
	\end{table}
\end{small}

From this table, we can see that our method works better than others.
As we see from Table.\ref{tab.6}, due to the fact that the figures for classification accuracy for almost all of these situations are just near 100,
we cannot perceive the effects of these 3 different methods very well. As a result of this, in order to show that our proposed method yields better results in comparison with the  GLRAM and the SVD,  we have done our experiments by another classifier, $\textit{K Nearest Neighbors}$ on PIE dataset.
In Figure.\ref{fig.5}, we demonstrate the result of 1-NN, 2-NN, and 3-NN classifiers with 2,5, and 10-fold cross-validation on PIE dataset respect to different values of $d=[2,...,16]$.
\begin{figure}[H]
	\begin{center}
		\begin{center}
		\includegraphics[scale=0.55]{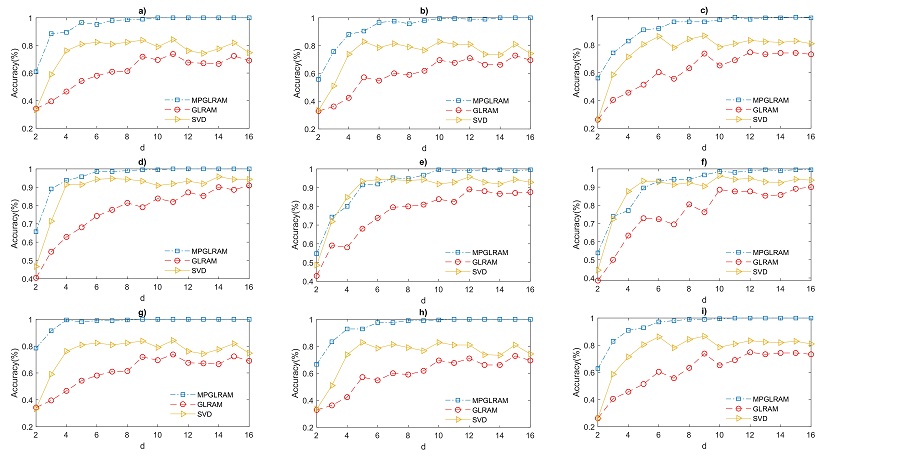}
		\end{center}
		\caption{{\footnotesize{ Comparison between SVD with rank-$d^2$, GLRAM, and MPGLRAM with $(d,d)$ ranks on PIE. a) 2-fold with 1-NN, b) 2-fold with 2-NN, c)2-fold with 3-NN, d) 5-fold with 1-NN, e) 5-fold with 2-NN f) 5-fold with 3-NN, g) 10-fold with 1-NN, h) 10-fold with 2-NN, and i) 10-fold with 3-NN.}}}\label{fig.5}
	\end{center}	
\end{figure}
From this table, the quality of the proposed MPGLRAM over other methods cloud be found clearly.

\subsection{$k$-Parameter}
In our proposed methods, we expanded the GLRAM search space by using k-pair projections. So, at first glance, this seems to play a vital role to achieve the best accuracy. While, as we have seen from the experimental results, by increasing the value of $k$ the RMSRE will be decreased as well, but in classification, the best accuracy occurred in different values of $k$. Also, even for small values of $k$ MPGLRAM gives results better than other methods. To show this issue we report the accuracy of the MPGLRAM for different datasets and different values of $d$ according to $k=2,3,4,5$ and 2-fold, 5-fold, and 10-fold cross-validation in Figure.\ref{fig.6}, Figure.\ref{fig.7}, and Figure.\ref{fig.8} respectively.
\begin{figure}[h]
	\centering
	\includegraphics[scale=.55]{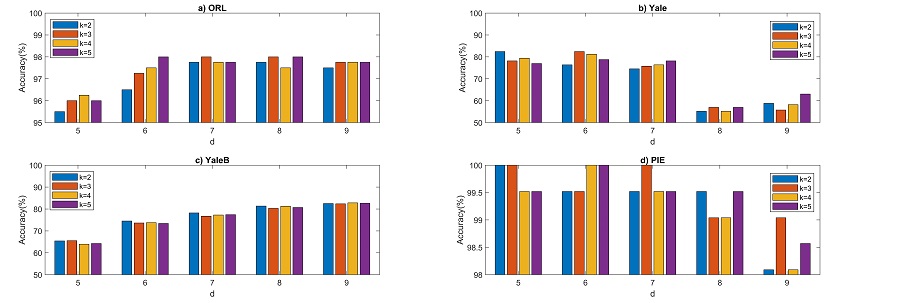}
	\caption{{\footnotesize{ The accuracy of MPGLRAM regarding to different values of $k=2,3,4,5$ on a) ORL, b) Yale, c) YaleB, and d) PIE. 
	}}}\label{fig.6}
\end{figure}
\begin{figure}[!h]
	\centering
	\includegraphics[scale=.55]{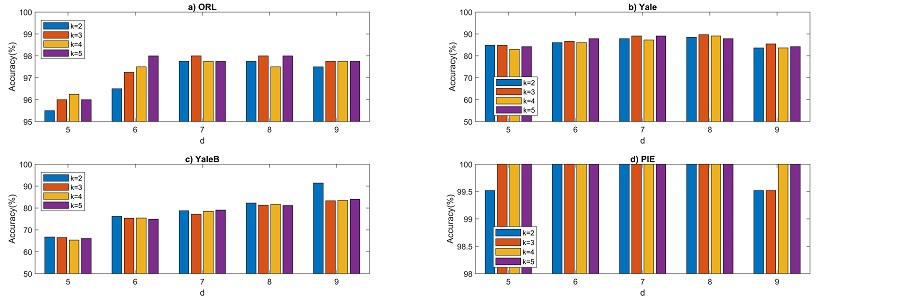}
	\caption{{\footnotesize{ The accuracy of MPGLRAM regarding to different values of $k=2,3,4,5$ on a) ORL, b) Yale, c) YaleB, and d) PIE. 
	}}}\label{fig.7}
\end{figure}
\begin{figure}[!h]
	\centering
	\includegraphics[scale=.55]{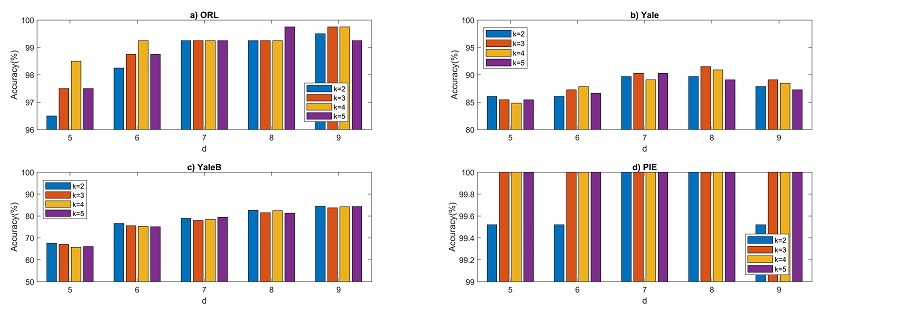}
	\caption{{\footnotesize{ The accuracy of MPGLRAM regarding to different values of $k=2,3,4,5$ on a) ORL, b) Yale, c) YaleB, and d) PIE. 
	}}}\label{fig.8}
\end{figure}

Here we see in different situations the best results obtained for small values of $k$. 

The value of $k$ can change from 1 to the minimum amount of the size of data. For example for a $m\times n $ matrix it can be $1\leq k \leq \min \lbrace m,n\rbrace$. When $k=1$ our method behaves like GLRAM, except its orthogonality constraints, so it has a small search space $m+n$. While when $k=\min \lbrace m,n\rbrace$, the search space is equal to the vectorized form of the matrix. So, in MPGLRAM we use $k$ to make a balance between these two methods. Therefore, in MPGLRAM, the search space will be $k(m+n)$  which can be larger than GLRAM and smaller than vectorized dimension reduction method, SVD.\\
In our experiments, we use $k=2,3,4,5$ to show that the proposed method works better than GLRAM especially when the dimension reduced to a lower value. An appropriate value for $k$ could be obtained by cross-validation approach.

\section{Conclusion}
In this paper, we proposed a novel method using the advantages of both SVD and GLRAM simultaneously to find a more accurate answer rather than GLRAM with lower complexity than SVD. This is done by k-pair of transformation in GLRAM method to enlarge its search space. By this method by few numbers of parameters which is impotant in the reduction of the possibility of over-fitting, we able to find results better than GLRAM and even SVD.
 The reported experimental results confirm the quality of the proposed method. Here we found that our method 
at the same time have the benefits of SVD and GLRAM methods and in fact, gives a trade-off between the size of search space(free parameters) and occurrence of overfitting and so almost gives better results in comparison with SVD and GLRAM.
 Also, this approach could be used on other multilinear methods like multilinear LDA.





\end{document}